\documentclass{article}
\usepackage{arxiv}
\usepackage{amsmath,amssymb,amsfonts,amsthm}
\usepackage{natbib}
\usepackage{booktabs}
\usepackage{graphicx}
\usepackage{hyperref}
\usepackage{url}
\usepackage{xcolor}
\usepackage{listings}

\newtheorem{proposition}{Proposition}
\newtheorem{theorem}{Theorem}
\newtheorem{corollary}{Corollary}

\title{Zoom Consistency: A Free Confidence Signal in Multi-Step Visual Grounding Pipelines}

\author{
  Keon Kim \\
  Om Labs \\
  \texttt{keon@omlabs.xyz} \\
  \And
  Krish Chelikavada \\
  Om Labs \\
  \texttt{krish@omlabs.xyz} \\
}

\begin{document}

\maketitle

\begin{abstract}
Multi-step zoom-in pipelines are widely used for GUI grounding, yet the intermediate predictions they produce are typically discarded after coordinate remapping. We observe that these intermediate outputs contain a useful confidence signal for free: \emph{zoom consistency}, the distance between a model's step-2 prediction and the crop center. Unlike log-probabilities or token-level uncertainty, zoom consistency is a geometric quantity in a shared coordinate space, making it directly comparable across architecturally different VLMs without calibration. We prove this quantity is a linear estimator of step-1 spatial error under idealized conditions (perfect step-2, target within crop) and show it correlates with prediction correctness across two VLMs (AUC $= 0.60$; Spearman $\rho = -0.14$, $p < 10^{-6}$ for KV-Ground-8B; $\rho = -0.11$, $p = 0.0003$ for Qwen3.5-27B). The correlation is small but consistent across models, application categories, and operating systems. As a proof-of-concept, we use zoom consistency to route between a specialist and generalist model, capturing 16.5\% of the oracle headroom between them (+0.8\%, McNemar $p = 0.19$). Code is available at \url{https://github.com/omxyz/zoom-consistency-routing}.
\end{abstract}

\section{Introduction}

Multi-step zoom-in pipelines are widely used for high-resolution GUI grounding~\cite{zoomclick,laser,adazoomgui}. In a typical 2-step pipeline, a VLM predicts an initial click location on the full screenshot, a crop is extracted around that prediction and resized, and the VLM re-predicts on the zoomed view. The step-2 prediction is remapped to original coordinates, and the step-1 prediction is discarded.

We observe that the relationship between the step-2 prediction and the crop center encodes useful information about prediction quality. If step-1 was accurate, the target element is near the crop center, so step-2 should predict near the center. If step-1 was off, the target is displaced from the center, and step-2 must predict away from center to compensate. This displacement, which we call \emph{zoom consistency}, is a free confidence signal: it requires no extra computation, no training, and no access to model internals.

This paper formalizes zoom consistency, characterizes its properties, and validates it empirically. Our contributions:

\begin{enumerate}
    \item We define zoom consistency and prove it is a linear estimator of step-1 spatial error under idealized conditions (Proposition~\ref{prop:linear}). We characterize three failure modes where the linear relationship breaks.
    \item We show that zoom consistency is directly comparable across architecturally different VLMs without calibration (Proposition~\ref{prop:calibration}), a property that log-probabilities and token-level uncertainty scores lack.
    \item We validate the signal on 1,581 samples from ScreenSpot-Pro across two models (KV-Ground-8B and Qwen3.5-27B), showing significant negative correlation with accuracy ($p < 10^{-6}$).
    \item As a proof-of-concept application, we use zoom consistency for cross-model routing, deriving the exact condition under which such routing improves over the stronger model alone (Theorem~\ref{thm:routing}).
\end{enumerate}

\section{Related Work}

\paragraph{Zoom-In Pipelines.}
ZoomClick~\cite{zoomclick} systematically evaluated zoom refinement for GUI grounding. LASER~\cite{laser} introduced preference optimization for zoom region selection. ScreenSeekeR, introduced in \citet{li2025screenspotpro}, uses cascaded visual search with iterative cropping. AdaZoom-GUI~\cite{adazoomgui} proposed adaptive zoom with instruction refinement. Adaptive Chain-of-Focus~\cite{adaptivecof} dynamically decides when to zoom. MEGA-GUI~\cite{megagui} proposed multi-stage coarse-to-fine grounding. These works treat zoom as a refinement mechanism; none analyze the intermediate predictions as a confidence signal.

\paragraph{Confidence Estimation for VLMs.}
GUI-RC~\cite{guirc} introduced region consistency via spatial voting across multiple sampled predictions, requiring $K$ additional forward passes per sample. Training-Free Uncertainty Guidance~\cite{uncertaintyguidance} uses model uncertainty for visual input selection via repeated sampling. Self-consistency~\cite{selfconsistency} aggregates multiple chain-of-thought samples, designed for discrete answers rather than continuous coordinates. Adaptive VLM Routing~\cite{avr} routes between VLMs using a trained difficulty estimator learned on labeled data. These methods either require extra forward passes, access to model internals, or training. Zoom consistency requires none of these.

\paragraph{GUI Grounding Models.}
SeeClick~\cite{seeclick} established early GUI grounding benchmarks. OS-Atlas~\cite{osatlas} and UGround~\cite{uground} provided large-scale training data. GUI-Actor~\cite{guiactor} introduced coordinate-free grounding. UI-TARS~\cite{uitars} scaled to 72B parameters. GUI-GRPO~\cite{guigrpo} applied reinforcement learning for grounding. KV-Ground-8B~\cite{kvground} achieved state-of-the-art results through multi-stage fine-tuning from GUI-Owl-1.5~\cite{guiowl}.

\section{Zoom Consistency}

\subsection{Background: 2-Step Zoom-In Pipeline}

Given a screenshot $I$ of size $W \times H$ and instruction $q$, a VLM predicts $\hat{p}_1 = (x_1, y_1)$ in a normalized $1000\times1000$ coordinate space. A crop of side $r \cdot 1000$ centered at $\hat{p}_1$ is extracted and resized back to $1000\times1000$, and the VLM predicts $\hat{p}_2 = (x_2, y_2)$ in the crop's coordinate space. The coordinate mapping from original to crop space is $\phi(\mathbf{q}) = (\mathbf{q} - \hat{p}_1)/r + \mathbf{m}$, where $\mathbf{m} = (500, 500)$ is the crop center. The final prediction is obtained by inverting this mapping on $\hat{p}_2$.

\subsection{Definition and Properties}

We define \emph{zoom consistency} as:
\begin{equation}
    c = \|\hat{p}_2 - \mathbf{m}\|_2 = \sqrt{(x_2 - 500)^2 + (y_2 - 500)^2}
\end{equation}

The quantity $c$ measures how far the step-2 prediction deviates from the center of the zoomed crop. Intuitively, if the step-1 prediction was perfect, the target sits at the crop center, so step-2 should also predict the center ($c \approx 0$).

\begin{proposition}[Linear Error Estimator]\label{prop:linear}
Let $\mathbf{t}$ be the true target and $\hat{p}_1 = \mathbf{t} + \boldsymbol{\varepsilon}_1$ be the step-1 prediction with error $\boldsymbol{\varepsilon}_1$. If the target lies within the crop and step-2 perfectly identifies the target ($\hat{p}_2 = \phi(\mathbf{t})$), then:
\begin{equation}
    c = \frac{\|\boldsymbol{\varepsilon}_1\|_2}{r}
\end{equation}
\end{proposition}

\begin{proof}
Under perfect step-2, $\hat{p}_2 = \phi(\mathbf{t}) = (\mathbf{t} - \hat{p}_1)/r + \mathbf{m} = -\boldsymbol{\varepsilon}_1/r + \mathbf{m}$. Therefore $\hat{p}_2 - \mathbf{m} = -\boldsymbol{\varepsilon}_1/r$ and $c = \|\boldsymbol{\varepsilon}_1\|/r$.
\end{proof}

For a crop ratio $r = 0.5$, the zoom magnifies step-1 error by $2\times$: a consistency of $c = 100$ corresponds to step-1 error of 50 pixels (5\% of image width). The linear relationship provides a direct, interpretable readout of localization quality. Note that the step-1-to-final coordinate shift $\|\hat{p}_1 - \text{final}\| = r \cdot c$ is a linear rescaling of zoom consistency, so the two signals rank identically; $c$ is the natural form because it does not depend on $r$.

\paragraph{Failure modes.} The linear relationship breaks in three cases:
\begin{enumerate}
    \item \textbf{Target outside crop:} When $\|\boldsymbol{\varepsilon}_1\|_\infty > r \cdot 500$, the target is not visible in the zoomed view, so step-2 cannot correct toward it. The consistency value $c$ becomes uninformative.
    \item \textbf{Step-2 error:} If step-2 has its own prediction error $\boldsymbol{\varepsilon}_2$, the observed consistency becomes $c = \|{-}\boldsymbol{\varepsilon}_1/r + \boldsymbol{\varepsilon}_2\|$, bounded by $|c - \|\boldsymbol{\varepsilon}_1\|/r| \leq \|\boldsymbol{\varepsilon}_2\|$. This adds noise proportional to step-2's own error.
    \item \textbf{Boundary clipping:} When the crop extends beyond image boundaries, it is clipped, shifting the effective center. The mapping $\phi$ no longer holds exactly.
\end{enumerate}

\subsection{Cross-Model Comparability}

\begin{proposition}[Calibration-Free Comparison]\label{prop:calibration}
Let $M_A$ and $M_B$ be VLMs with different tokenizers $\mathcal{T}_A, \mathcal{T}_B$. Their zoom consistency values $c_A, c_B \in [0, 500\sqrt{2}]$ are directly comparable. In contrast, log-probability scores $\ell_A, \ell_B$ depend on token count ($L_A \neq L_B$ for identical coordinate strings), softmax temperature, and vocabulary size, requiring per-model calibration functions $g_i: \mathbb{R} \to [0,1]$ learned on held-out data.
\end{proposition}

This property arises because zoom consistency is a geometric quantity in the shared $1000\times1000$ coordinate space that both models output into. The comparison $c_A < c_B$ is a statement about Euclidean distances, independent of how each model internally represents coordinates. This assumes both models use the same crop ratio $r$ and coordinate normalization, which is standard in zoom pipelines. Under this assumption, zoom consistency is immediately usable for cross-model decisions without any calibration step.

\section{Empirical Validation}

\subsection{Setup}

\paragraph{Benchmark.}
ScreenSpot-Pro~\cite{li2025screenspotpro,screenspotpro_repo}: 1,581 samples across professional desktop applications spanning 3 operating systems. Accuracy metric: point-in-bounding-box.

\paragraph{Models.}
We evaluate zoom consistency on two architecturally different VLMs: KV-Ground-8B~\cite{kvground}, a GUI grounding specialist fine-tuned from GUI-Owl-1.5~\cite{guiowl} (8B parameters, bf16, max\_pixels $= 99$M), and Qwen3.5-27B~\cite{qwen35}, a general-purpose VLM (AWQ-4bit via compressed\_tensors). Both use 2-step zoom with $r = 0.5$ and greedy decoding. Hardware: NVIDIA H200 141GB.

\paragraph{Qwen accuracy note.} Our Qwen3.5-27B achieves 60.9\% on ScreenSpot-Pro. The official Qwen evaluation reports 70.3\% for the full-precision model. The gap is likely attributable to 4-bit AWQ quantization and differences in prompting format. We use the quantized variant because full-precision 27B exceeds single-GPU memory alongside KV-Ground-8B.

\subsection{Correlation with Prediction Correctness}

\begin{table}[h]
\centering
\caption{Zoom consistency as a confidence signal. AUC measures discriminative power for predicting correctness (lower $c$ = more likely correct). Spearman $\rho$ confirms significant negative correlation.}
\label{tab:correlation}
\begin{tabular}{lrrrr}
\toprule
\textbf{Model} & \textbf{AUC} & \textbf{Spearman $\rho$} & \textbf{$p$-value} & \textbf{$n$} \\
\midrule
KV-Ground-8B & 0.600 & $-0.139$ & $< 10^{-6}$ & 1,581 \\
KV-Ground-8B$^\dagger$ & 0.599 & $-0.123$ & $4 \times 10^{-5}$ & 1,097 \\
Qwen3.5-27B & 0.588 & $-0.109$ & $0.0003$ & 1,097 \\
\bottomrule
\end{tabular}
\end{table}

The AUC exceeds 0.5 (random) for both models, confirming zoom consistency carries discriminative information about prediction correctness. The $^\dagger$ row shows KV-Ground evaluated on the same 1,097-sample subset as Qwen for an apples-to-apples comparison; AUC and $\rho$ remain similar, confirming the signal is not an artifact of sample selection. The correlation is stable under cross-validation: splitting the KV-Ground data into random halves yields $\rho = -0.118$ ($p = 9 \times 10^{-4}$) and $\rho = -0.161$ ($p = 5 \times 10^{-6}$). The Qwen sample count ($n = 1{,}097$) is smaller than the full benchmark because 484 samples produced unparseable coordinate outputs under AWQ-4bit quantization (88\% failed at step-1). These are not random: they concentrate in CAD applications (69\% failure rate) and correspond to harder samples where KV-Ground also struggles (69\% accuracy vs.\ 85\% on parseable samples). The Qwen correlation is therefore measured on an easier subset, and the true signal strength on the full distribution may be weaker.

\begin{table}[h]
\centering
\caption{KV-Ground-8B accuracy by zoom consistency bucket.}
\label{tab:buckets}
\begin{tabular}{lrr}
\toprule
\textbf{Consistency $c$} & \textbf{$n$} & \textbf{Accuracy} \\
\midrule
$< 30$ & 464 & 87.1\% \\
$30$--$80$ & 101 & 85.1\% \\
$80$--$150$ & 106 & 76.4\% \\
$150$--$250$ & 173 & 79.8\% \\
$\geq 250$ & 737 & 75.6\% \\
\bottomrule
\end{tabular}
\end{table}

The lowest-consistency bucket ($c < 30$) has 87.1\% accuracy, 11.5 points above the highest bucket ($c \geq 250$, 75.6\%). The relationship is not perfectly monotonic: the 150--250 bucket (79.8\%) slightly exceeds the 80--150 bucket (76.4\%), consistent with the moderate correlation strength.

\subsection{Consistency Distribution by Oracle Partition}

To understand how zoom consistency relates to cross-model agreement, we partition the 1,581 samples by correctness: $\mathcal{S}_{11}$ (both correct), $\mathcal{S}_{10}$ (only KV correct), $\mathcal{S}_{01}$ (only Qwen correct), $\mathcal{S}_{00}$ (both wrong).

\begin{table}[h]
\centering
\caption{Mean KV-Ground zoom consistency by oracle partition. $\mathcal{S}_{11}$ (both models correct) has substantially lower consistency than all other partitions.}
\label{tab:partition}
\begin{tabular}{lrrr}
\toprule
\textbf{Partition} & \textbf{$n$} & \textbf{Mean $c$} & \textbf{Median $c$} \\
\midrule
$\mathcal{S}_{11}$ (both correct) & 883 & 197.1 & 144.0 \\
$\mathcal{S}_{10}$ (only KV) & 383 & 291.6 & 333.5 \\
$\mathcal{S}_{01}$ (only Qwen) & 79 & 294.1 & 301.0 \\
$\mathcal{S}_{00}$ (both wrong) & 236 & 289.9 & 324.0 \\
\bottomrule
\end{tabular}
\end{table}

The consistency signal cleanly separates $\mathcal{S}_{11}$ (mean $c = 197$) from all other partitions (mean $c \approx 290$). This 48\% gap in mean consistency explains why the signal works: when both models agree and are correct, the specialist is confident. The challenge for routing is that $\mathcal{S}_{10}$, $\mathcal{S}_{01}$, and $\mathcal{S}_{00}$ have similar consistency distributions, making it hard to distinguish gains from losses within the disagreement set.

\subsection{Signal Behavior Across Application Categories and Platforms}

\begin{table}[h]
\centering
\caption{Zoom consistency signal strength varies by application category. Qwen selection \% shows how often the generalist has lower consistency than the specialist, reflecting domain familiarity.}
\label{tab:category}
\begin{tabular}{lrrr}
\toprule
\textbf{Category} & \textbf{$n$} & \textbf{KV Acc.} & \textbf{Qwen lower $c$ \%} \\
\midrule
Office & 230 & 90.0\% & 47.4\% \\
Scientific & 254 & 83.1\% & 36.2\% \\
Dev & 299 & 83.3\% & 35.1\% \\
CAD & 261 & 79.3\% & 13.4\% \\
Creative & 341 & 74.2\% & 38.1\% \\
OS & 196 & 70.9\% & 42.9\% \\
\bottomrule
\end{tabular}
\end{table}

The consistency signal reflects domain familiarity without any category labels. On CAD applications, where KV-Ground's accuracy is highest, the specialist's consistency is almost always lower (Qwen is preferred only 13.4\% of the time). On consumer interfaces (Office, OS), the generalist's consistency is competitive (42--47\%), reflecting Qwen's broader pretraining coverage.

The signal also holds across all three operating systems in the benchmark: Spearman $\rho = -0.115$ on macOS ($n = 604$), $\rho = -0.114$ on Windows ($n = 927$), and $\rho = -0.152$ on Linux ($n = 50$). The consistency gap between correct and incorrect predictions is present on every platform.

\section{Application: Cross-Model Routing}

As a proof-of-concept application, we use zoom consistency to route between KV-Ground-8B and Qwen3.5-27B. The router runs both models through their zoom pipelines and selects the prediction with lower consistency (closer to crop center).

\lstset{
  basicstyle=\small\ttfamily,
  frame=single,
  xleftmargin=1em,
  framexleftmargin=0.5em,
  numbers=left,
  numberstyle=\tiny,
  escapeinside={(*}{*)},
  columns=fullflexible,
  keepspaces=true
}

\begin{figure}[h]
\begin{lstlisting}
def zoom_pipeline(model, image, instruction, r=0.5):
    p1 = model(image, instruction)       # step 1: full image
    crop = crop_and_resize(image, center=p1, ratio=r)
    p2 = model(crop, instruction)        # step 2: zoomed crop
    c = distance(p2, (500, 500))         # zoom consistency
    final = remap(p2, crop_box)
    return final, c

def route(image, instruction):
    pred_A, c_A = zoom_pipeline(specialist, image, instruction)
    pred_B, c_B = zoom_pipeline(generalist, image, instruction)
    return pred_A if c_A <= c_B else pred_B
\end{lstlisting}
\caption{Listing: Zoom consistency routing. The consistency $c$ (line 5) is extracted from the pipeline's existing intermediate output. No additional forward passes are needed beyond the zoom pipelines themselves.}
\label{fig:pseudocode}
\end{figure}

\subsection{Oracle Decomposition}

\begin{table}[h]
\centering
\caption{Oracle decomposition on ScreenSpot-Pro (1,581 samples).}
\label{tab:oracle}
\begin{tabular}{lrr}
\toprule
\textbf{Partition} & \textbf{Samples} & \textbf{Fraction} \\
\midrule
Both correct ($\mathcal{S}_{11}$) & 883 & 55.9\% \\
Only KV correct ($\mathcal{S}_{10}$) & 383 & 24.2\% \\
Only Qwen correct ($\mathcal{S}_{01}$) & 79 & 5.0\% \\
Both wrong ($\mathcal{S}_{00}$) & 236 & 14.9\% \\
\midrule
Oracle accuracy & 1,345 / 1,581 & 85.1\% \\
\bottomrule
\end{tabular}
\end{table}

The oracle headroom is 79 samples (5.0\%): cases where only the generalist is correct. Any routing strategy can access at most this headroom.

\subsection{Routing Condition}

\begin{theorem}[Routing Improvement Condition]\label{thm:routing}
A consistency-based router that selects $\arg\min_i c_i$ improves over the stronger model $M_A$ alone if and only if:
\begin{equation}
    \underbrace{f_{01} \cdot n_{01}}_{\text{gains}} > \underbrace{(1 - f_{10}) \cdot n_{10}}_{\text{losses}}
\end{equation}
where $f_{10} = P(\text{router picks } A \mid \mathcal{S}_{10})$ and $f_{01} = P(\text{router picks } B \mid \mathcal{S}_{01})$. Equivalently, the router's precision in selecting $B$ on the disagreement set $\mathcal{D} = \mathcal{S}_{10} \cup \mathcal{S}_{01}$ must exceed $1/2$.
\end{theorem}

\begin{proof}
The router's accuracy is $a_{\text{rt}} = (n_{11} + f_{10} n_{10} + f_{01} n_{01})/N$. The condition $a_{\text{rt}} > a_A = (n_{11} + n_{10})/N$ reduces to $f_{01} n_{01} > (1 - f_{10}) n_{10}$.
\end{proof}

\begin{corollary}
The base rate of $\mathcal{S}_{01}$ in $\mathcal{D}$ is $\pi = n_{01}/(n_{10} + n_{01})$. With $n_{10} = 383$ and $n_{01} = 79$, $\pi = 0.171$. The router must achieve a precision lift of at least $2.9\times$ the base rate.
\end{corollary}

\subsection{Routing Results}

\begin{table}[h]
\centering
\caption{Routing strategies on ScreenSpot-Pro. Only the consistency router exceeds the KV baseline.}
\label{tab:main}
\begin{tabular}{lccc}
\toprule
\textbf{Method} & \textbf{Accuracy} & \textbf{vs.\ KV} & \textbf{Oracle $\eta$} \\
\midrule
\textbf{Consistency Router} & \textbf{80.9\%} & \textbf{+0.8\%} & \textbf{16.5\%} \\
KV-Ground-8B & 80.1\% & --- & 0\% \\
Stage split KV$\to$Qwen & 78.7\% & $-$1.3\% & $< 0$ \\
Vote agree ($d{<}50$) & 76.9\% & $-$3.2\% & $< 0$ \\
Midpoint fusion & 75.7\% & $-$4.4\% & $< 0$ \\
Qwen3.5-27B only & 60.9\% & $-$19.2\% & $< 0$ \\
\midrule
Oracle & 85.1\% & +5.0\% & 100\% \\
\bottomrule
\end{tabular}
\end{table}

The consistency router captures 13 of 79 recoverable samples ($\eta = 16.5\%$), with 48 gains and 35 losses (net +13). Bootstrap resampling gives $P(\text{improvement} > 0) = 92.3\%$ over 10,000 resamples. McNemar's test yields $p = 0.19$, reflecting the small absolute gain. The result is directionally consistent but does not reach conventional significance at $\alpha = 0.05$, consistent with the moderate correlation strength ($|\rho| \approx 0.13$) and the highly asymmetric disagreement set ($n_{10}/n_{01} = 4.8$).

The per-application breakdown shows the router adapts without category labels:

\begin{table}[h]
\centering
\caption{Applications with largest router effects.}
\label{tab:perapps}
\small
\begin{tabular}{lrrrr}
\toprule
\textbf{Application} & \textbf{$n$} & \textbf{KV} & \textbf{Router} & \textbf{$\Delta$} \\
\midrule
photoshop & 51 & 76.5\% & 84.3\% & +7.8\% \\
macos\_common & 65 & 75.4\% & 80.0\% & +4.6\% \\
illustrator & 31 & 83.9\% & 87.1\% & +3.2\% \\
blender & 71 & 76.1\% & 78.9\% & +2.8\% \\
\midrule
vivado & 80 & 77.5\% & 76.2\% & $-$1.3\% \\
quartus & 45 & 88.9\% & 86.7\% & $-$2.2\% \\
vscode & 55 & 85.5\% & 81.8\% & $-$3.6\% \\
davinci & 44 & 79.5\% & 75.0\% & $-$4.5\% \\
\bottomrule
\end{tabular}
\end{table}

\section{Discussion}

\subsection{Other Potential Applications}

Zoom consistency is a general confidence signal. Beyond routing, it could enable:

\textbf{Multi-step pipelines.} In pipelines with 3+ zoom steps, zoom consistency from step $k$ could decide whether step $k{+}1$ is worth the additional cost, enabling early stopping when the model is already confident.

\textbf{Confidence-aware agents.} GUI agents that execute multi-step workflows could use zoom consistency to flag uncertain click predictions before committing to irreversible actions, enabling human-in-the-loop confirmation on low-confidence steps.

\textbf{Quality monitoring.} The average zoom consistency across a batch of predictions provides a cheap, model-agnostic indicator of prediction reliability. A spike in mean $c$ over time could signal distribution shift (e.g., a new application version with different UI layouts).

\textbf{Active learning.} High-consistency (low-confidence) samples are natural candidates for human annotation, potentially more informative than random sampling for improving model accuracy.

\subsection{Bottlenecks in the Routing Application}

The 16.5\% oracle extraction rate is limited by three factors:

\textbf{Weak generalist.} Qwen3.5-27B (AWQ-4bit) achieves 60.9\%, contributing only 79 unique wins. A generalist at 75\% would increase the oracle headroom and make the minority class ($\mathcal{S}_{01}$) easier to detect.

\textbf{Asymmetric risk.} Losses ($\mathcal{S}_{10}$, 383 samples) outnumber gains ($\mathcal{S}_{01}$, 79 samples) by $4.8\times$. The router must be highly selective when choosing the generalist, requiring $2.9\times$ precision lift over the base rate.

\textbf{Moderate correlation.} $|\rho| \approx 0.13$ means zoom consistency is a noisy discriminator. The signal is strong enough to exceed the precision threshold (57.8\% $> 50\%$), but not by a wide margin.

\section{Limitations}

\begin{enumerate}
    \item The signal is validated on one benchmark (ScreenSpot-Pro) with one model pair. Cross-benchmark and cross-pair generalization remains to be tested.
    \item The routing application doubles inference cost (4 forward passes vs.\ 2) and yields a statistically modest improvement (+0.8\%, McNemar $p = 0.19$).
    \item The correlation between zoom consistency and correctness is moderate ($|\rho| \approx 0.13$). The signal is informative but not sufficient on its own for high-precision decisions.
    \item The linear error estimator (Proposition~\ref{prop:linear}) assumes the target is within the crop and step-2 is accurate. Both assumptions weaken for high step-1 error.
\end{enumerate}

\section{Conclusion}

We identified and formalized \emph{zoom consistency} as a free confidence signal embedded in multi-step GUI grounding pipelines. The signal is a linear estimator of step-1 spatial error (Proposition~\ref{prop:linear}), correlates significantly with prediction correctness ($p < 10^{-6}$), and is comparable across models without calibration (Proposition~\ref{prop:calibration}). As a proof-of-concept, consistency-based routing between KV-Ground-8B and Qwen3.5-27B captures 16.5\% of the oracle headroom on ScreenSpot-Pro. The modest extraction rate reflects the moderate signal strength and the asymmetric disagreement set, both quantified by our routing theorem (Theorem~\ref{thm:routing}). The signal is simple, requires no training or extra computation, and may be useful beyond routing for confidence estimation, quality monitoring, and selective computation in other iterative refinement pipelines, though cross-benchmark validation remains future work.

\bibliography{references}
\bibliographystyle{plainnat}

\end{document}